\documentclass[11pt]{article}
\usepackage[utf8]{inputenc}
\usepackage[margin=1in,footskip=0.25in]{geometry}
\usepackage{amsmath,amsthm}
\usepackage{enumerate}
\usepackage{fancyhdr}
\usepackage{color}
\usepackage{listings}
\usepackage{graphicx}
\usepackage{amssymb}
\usepackage{mathtools}
\usepackage{comment}
\usepackage{hyperref}
\usepackage{subfig}
\usepackage[ruled,vlined]{algorithm2e}

\newtheorem{theorem}{Theorem}
\newtheorem{fact}{Fact}
\newtheorem{corollary}{Corollary}[theorem]
\newtheorem{definition}{Definition}
\newtheorem{lemma}[theorem]{Lemma}
\newtheorem{proposition}{Proposition}

\newtheorem{remark}[theorem]{Remark}
\newcommand{\norm}[1]{\left\lVert#1\right\rVert}

\definecolor{lightgray}{gray}{0.5}

\newcommand*\bI{\mathbf{I}}
\newcommand*\bSigma{\mathbf{\Sigma}}
\newcommand*\E{\mathbb{E}}
\newcommand*\R{\mathbb{R}}

\newcommand*\x{\mathbf{x}}

\newcommand*\bu{\mathbf{u}}
\newcommand*\ba{\mathbf{a}}

\newcommand*\A{\mathbf{A}}
\newcommand*\J{\mathbf{J}}

\newcommand{\eqdef}{\mathbin{\stackrel{\rm def}{=}}}
\makeatletter
\def\hlinewd#1{%
	\noalign{\ifnum0=`}\fi\hrule \@height #1 \futurelet
	\reserved@a\@xhline}
\makeatother

\newcommand{\bs}[1]{\boldsymbol{#1}}
\newcommand{\bv}[1]{\mathbf{#1}}

\newcommand{\Cam}[1]{\textcolor{blue}{Cam: #1}}

\title{Estimation of Shortest Path Covariance Matrices}
\author{Raj Kumar Maity \\ UMass Amherst \\ \texttt{rajkmaity@cs.umass.edu} \and Cameron Musco \\ UMass Amherst\\ \texttt{cmusco@cs.umass.edu}}
\date{}

\begin{document}

\maketitle


\begin{abstract}
We study the sample complexity of estimating the covariance matrix $\bSigma \in \R^{d\times d}$ of a distribution $\mathcal D$ over $\R^d$ given independent samples, under the assumption that $\bSigma$ is \emph{graph-structured}. In particular, we focus on shortest path covariance matrices, where the covariance between any two measurements is determined by the shortest path distance in an underlying graph with $d$ nodes. Such matrices generalize Toeplitz and circulant covariance matrices and are widely applied in signal processing applications, where the covariance between two measurements depends on the (shortest path) distance between them in time or space.

We focus on minimizing both the \emph{vector sample complexity}: the number of samples drawn from $\mathcal{D}$ and the \emph{entry sample complexity}: the number of entries read in each sample. The entry sample complexity corresponds to measurement equipment costs in signal processing applications. We give a very simple algorithm for estimating $\bs \Sigma$ up to spectral norm error $\epsilon \norm{\bSigma}_2$ using just $O(\sqrt{D})$ entry sample complexity and $\tilde O(r^2/\epsilon^2)$ vector sample complexity, where $D$ is the diameter of the underlying graph and $r \le d$ is the rank of $\bSigma$.
Our method is based on extending the widely applied idea of \emph{sparse rulers} for Toeplitz covariance estimation to the graph setting. 

In the special case when $\bSigma$ is a low-rank Toeplitz matrix, our result 
matches the state-of-the-art, with a far simpler proof. We also give an information theoretic lower bound matching our upper bound up to a factor $D$ and discuss some directions towards closing this gap. 

\end{abstract}

\section{Introduction}
Estimating the covariance matrix $\bSigma \in \R^{d \times d}$ of a distribution $\mathcal{D}$ over $\R^d$ from independent samples $\x^{(1)},\ldots ,\x^{(n)} \in \mathbb{R}^d$ is a fundamental statistical problem \cite{anderson1962introduction,kay1981spectrum,burg1982estimation}. When $\mathcal{D}$ is e.g., a multivariate sub-Gaussian distribution, it is known that $\Theta(d/\epsilon^2)$ samples are necessary and sufficient to estimate $\bSigma$ up to $\epsilon$ error in the spectral norm with high probability \cite{vershynin2018high} using the empirical covariance matrix $\frac{1}{n} \sum_{i=1}^n \x^{(i)} {\x^{(i)}}^T$.

When the dimension $d$ is large, this sample complexity can be high, and significant work has focused on giving improved bounds by leveraging \emph{structure} arising in various applications. Many types of structure have been considered, including when $\bSigma$ is low-rank \cite{chen2015exact}, sparse \cite{el2008operator}, low-rank plus sparse \cite{richard2012estimation,chen2014estimation}, and banded or generally decaying away from the diagonal \cite{bickel2008regularized}. Significant work focuses on the case when $\bSigma$ is Toeplitz, a structure that arises in signal processing and spatial statistics, when the covariance between measurements depends on the distances between them in time or space
\cite{burg1982estimation,barton1997structured,chen2015exact,qiao2017gridless,toep,lawrence2019low}.

\subsection{Graph-Structured Covariance Estimation}

In this work we focus on approximating \emph{graph-structured covariance matrices}. Graph structure has been used to describe correlations within complex data-sets in many domains, including gene regulatory networks \cite{bionetwork}, transportation networks \cite{trans}, and social and economic networks \cite{social}.  The entire field of \emph{graph signal processing} focuses on analyzing data with covariance structure based on an underlying graph \cite{shuman2013emerging,ortega2018graph}. Most commonly, the literature considers vectors whose covariance matrix  is of the form $\bSigma = \bv U \bs \Lambda \bv U^T$, where the columns of $\bv U \in \R^{d \times d}$ are the eigenvectors of some underlying graph Laplacian and $\bs \Lambda$ is any positive diagonal matrix, whose entries are known as the `graph power spectrum' \cite{perraudin2017stationary,marques2017stationary}. Covariance estimation under this model has been studied extensively \cite{graph,sandryhaila2014big,marques2017stationary}, although generally not in the style of giving formal sample complexity bounds for estimation up to a certain error.

Graph structured covariance matrices also arise in the study of Gaussian Markov random fields, where the inverse covariance $\bSigma^{-1}$ (the precision matrix) of a multivariate normal distribution is assumed to be sparse, with nonzero entries corresponding to the edges in some underlying dependency graph \cite{rue2005gaussian,uhler2017gaussian}. These edges correspond to conditionally dependent variables. Again, covariance estimation has been studied extensively in this model
\cite{ravikumar2011high,cai2011constrained,johnson2012high}.
%
%

\subsection{Shortest Path Covariance Matrices}
In this work we consider a simple combinatorial notion of graph-structure. We assume that each entry of $\bv x \sim \mathcal{D}$ corresponds to a node in some underlying graph $G$ and the covariance between entries $i$ and $j$, $\bSigma_{i,j}$ depends only on the \emph{shortest path distance} between the corresponding vertices $G$. Formally we define a {shortest path covariance matrix} as:
\begin{definition}[Shortest Path Covariance Matrix]\label{def:graphStructure}
A positive semidefinite matrix $\mathbf{\Sigma} \in \mathbb{R}^{d\times d}$ is said to be a shortest path covariance matrix if there is some unweighted graph $G = (V,E)$ with $d$ nodes and diameter $D$, along with a vector $\bv{a} = [\bv{a}_0,\ldots, \bv{a}_D]$ such that for all $i,j \in [d]$, $\bs{\Sigma}_{i,j} = \bv{a}_{d(v_i,v_j)}$, where $d(v_i,v_j)$ is the shortest path distance between $v_i$ and $v_j$ in $G$.
\end{definition}

Shortest path covariance matrices arise in many applications when the covariance between measurements depends on the distance between them. The most important example is when $G$ is a path graph. In this case, $\bs{\Sigma}$ is Toeplitz, with $\bs{\Sigma}_{i,j}$ depending only on $d(v_i,v_j) = |i-j|$. As discussed, Toeplitz covariance matrices are widely employed in signal processing, with applications ranging from signal direction of arrival estimation \cite{doa,5g,massive}, to sprectum sensing  \cite{radio1,radio2}, to medical and radar imaging \cite{img1,img2,img3,img4,img5}.

 When $G$ is a cycle, $\bs{\Sigma}$ is a circulant  matrix, a special case of a Toeplitz matrix with `wrap-around' structure \cite{gray2006toeplitz}. When $G$ is a grid or torus, $\bs{\Sigma}$ is a multi-dimensional Toeplitz or circulant matrix, a covariance structure arising in signal processing on two dimensional grids  \cite{yang2016vandermonde}.
 
 More general graph structure arises in fields such as ecology and geostatistics, when the shortest path distance over an appropriate graph is found to more effectively reflect the covariance of measurements than simple measures like Euclidean distance \cite{little1997kriging,kurtz2015sparse}. For example,  \cite{arnaud2003metapopulation} uses shortest path distances over a geographically-determined graph to study genetic variance among separated populations. \cite{gardner2003predicting} models the covariances between measurements of stream temperatures as depending on the shortest path distance between these measurements along a graph whose  edges correspond to paths in the river system. Similarly, shortest path covariance matrices are used for Gaussian process regression (kriging) when modeling phenomena in estuaries and other aquatic environments, where distances along a waterway graph reflect covariance between measurements better than Euclidean distance \cite{rathbun1998spatial,boisvert2011modeling}.

\subsection{Our Contributions}

As far as we are aware, outside the special case of Toeplitz matrices (path graphs), sample complexity bounds for estimating shortest path covariance matrices have not been studied formally. We study such bounds, focusing on two notions of sample complexity, which have been considered also in the study of Toeplitz covariance estimation \cite{chen2015exact,qiao2017gridless,toep}:
%
\begin{enumerate}
    \item \textbf{Vector Sample complexity (VSC)} The number of independent samples from $\bv x^{(1)},\ldots,\bv x^{(n)} \sim \mathcal{D}$ required to estimate $\bSigma$ to a specified tolerance.
    \item \textbf{Entry Sample complexity (ESC) } The number of entries in each $d$-dimensional sample $\bv{x}^{(i)} \sim \mathcal{D}$,  an algorithm must read, assuming the same subset of entries is read in each sample.
\end{enumerate}
VSC is the classic notion of sample complexity in high-dimensional statistics, while ESC complexity has received less attention, since for many covariance matrices (e.g., even when $\bs{\Sigma}$ is diagonal or rank-$1$), one cannot achieve ESC less than the trivial $d$. However, surprisingly when $\bSigma$ is Toeplitz, much lower ESC of $O(\sqrt{d})$ or $O(\sqrt{rank(\bSigma})$ can be achieved \cite{qiao2017gridless,toep,lawrence2019low}. In a signal processing setting, this complexity corresponds to the number of sensing devices needed, while VSC corresponds to the number of measurements that must be taken. In different applications, these different complexities have different costs and generally there is a tradeoff between achieving small ESC and small VSC.

Our main result shows that for a general graph $G$ we can achieve ESC complexity depending on the \emph{square root of $G$'s diameter}, while simultaneously achieving reasonable vector sample complexity, depending polynomially on the rank $d$ and the error parameter $\epsilon$. 
\begin{theorem}[Main Theorem -- Informal]\label{thm:mainIntro}
For any shortest path covariance matrix $\bSigma \in \R^{d \times d}$ (Def. \ref{def:graphStructure}),  there is an efficient algorithm which, given $n$ independent samples from the multivariate Gaussian distribution $\mathcal{N}(\bv 0,\bSigma)$, returns $\Tilde{\bSigma} $ such that $\|\bSigma- \Tilde{\bSigma}\|_2 \leq \epsilon\|\bSigma\|_2 $ with probability $\ge 1-\delta$, using entry sample complexity $O(\sqrt{D})$  and vector sample complexity $n = O \left (\frac{d^2\log (\frac{D}{\delta})}{\epsilon^2} \right )$, where $D \le d-1$ is the diameter of the underlying graph $G$.
\end{theorem}
Our entry sample complexity $O(\sqrt{D})$ is significantly less than the naive bound of $d$, and matches the start-of-the-art for the special case of Toeplitz matrices, where $D = d-1$ ($G$ is a path graph). It is easily seen to be optimal since reading $o(\sqrt{D})$ entries will only give access to $o(D)$ possible covariance values (between all pairs of entries read), and so we will receive no information about some entries in $\bSigma$.

We generalize our bound to when $\bSigma$ has rank-$k$, in which case $O \left (\frac{k^2\log (\frac{D}{\delta})}{\epsilon^2} \right )$ VSC can be achieved. This slightly improves on the best known result for Toeplitz covariance matrices \cite{toep}, while at the same time using a much simpler  proof that relies only on shortest path structure.
We show that our vector sample complexity is also near optimal given $O(\sqrt{D})$ ESC: 
\begin{theorem}[Lower Bound -- Informal]
Any algorithm that is given  $n$ independent samples from the multivariate Gaussian distribution $\mathcal{N}(\bv 0,\bSigma)$ where $\bSigma$ is a shortest path covariance matrix whose underlying graph has diameter $D$, reads a fixed subset of $O(\sqrt{D})$ entries from each sample, and returns $\bv{\tilde \Sigma}$ such that $\|\bSigma- \Tilde{\bSigma}\|_2 \leq \epsilon\|\Tilde{\bSigma}\|_2 $ with good probability requires $n = \Omega \left (\frac{d^2}{D\epsilon^2} \right )$ vector samples.
\end{theorem}

Our work leaves open the  question of if we can close the gap of $D$ in our vector sample complexity upper and lower bounds. \cite{toep} achieves this in the special case of Toeplitz matrices, where $D = d-1$. They give an algorithm achieving ESC $\Theta(\sqrt{d})$ and VSC $\tilde O \left ( \frac{d}{\epsilon^2} \right)$. We extend this to certain restricted types of trees, leaving open further extensions. 

\subsection{Our Techniques}

Our primary technique in achieving ESC $O(\sqrt{D})$ for shortest path covariance estimation is a simple extension of  the idea of \emph{sparse rulers} to the graph setting. A sparse ruler is a subset of integers $R \subseteq \{1,\ldots, d\}$ so that every distance in $\{0,\ldots, d-1\}$ can be `measured' as the the distance $|i-j|$ between some pair $i,j \in R$ \cite{leech1956representation,wichmann2002note}. It is not hard to see that one can construct a ruler $R$ with just $O(\sqrt{d})$ entries (see Proposition \ref{prop:sparse}). Using such a ruler, one can read just $O(\sqrt{d})$ entries of $\bv{x} \sim \mathcal{D}$ with Toeplitz covariance $\bSigma$ and still obtain an estimate of all entries in  $\bSigma$. Since $\bSigma_{k,\ell}$ depends only on the distance $|k-\ell| \in \{0,\ldots,d-1\}$, at least one pair of entries read $\bv x_i, \bv x_j$ (determined by the ruler $R)$ will have the same distance and hence covariance as entries $k$ and $\ell$. So $\bv x_i \cdot \bv x_j$ will be an unbiased estimator of $\bSigma_{k,\ell}$.
This simple but ingenious observation lets sparse rulers to be used to achieve $O(\sqrt{d})$ ESC for Toeplitz covariance estimation, and generally to improve the ESC of many signal processing problems in both theory and practice \cite{moffet1968minimum,pillai1985new,romero2015compressive,qiao2017gridless,toep}.

We show how to extend the notion of a sparse ruler to the graph setting, selecting a small subset of $O(\sqrt{D})$ nodes from $G$ whose pairwise shortest path distances include all distances from $0$ to the diameter $D$. By reading entries of $\bv{x} \sim \mathcal{D}$ at the indices corresponding to these nodes, we can thus obtain estimates of all entries of a shortest path covariance matrix $\bSigma$ with underlying graph $G$. With enough vector samples, these estimates are accurate enough to give a good estimate $\bs{\tilde \Sigma}$ of the full covariance matrix, yielding our main Theorem \ref{thm:mainIntro}.


\section{Background and Problem Formulation }

We start by introducing necessary notation and technical tools and formally defining our covariance estimation setting.

\subsection{Notation and Technical Tools} Throughout we use boldface  $\mathbf{X}$ to denote a matrix and boldface $\mathbf{x}$ to denote a vector. We let $\bv{X}_{i,j}$ denote the entry of $\bv{X}$ in the $i^{th}$ row and $j^{th}$ column and $\bv{x}_i$ denote the $i^{th}$ entry of $\bv{x}$. We use $\left\| \bv{X}\right\|_2$ to denote the matrix  spectral norm and for square $\mathbf{X}$, we use $det(\mathbf{X})$ to denote the determinant. We use $Tr(\mathbf{X})$ to denote the trace, which is the sum of the diagonal entries of $\bv{X}$, or equivalently, the sum of $\bv{X}$'s eigenvalues.
For any two distributions $P$ and $Q$ on any domain $\mathcal{X}$, we let $\left\|P-Q \right\|_{TV}$ denote the total variation distance $KL(P\|Q)$ denote the KL-divergence:  $KL(P\|Q) \eqdef \int_{x \in \mathcal{X}}P(x)\log\left(\frac{P(x)}{Q(x)} \right)dx$. 
In our sample complexity bounds we use a few standard probability tools related to concentration bounds for sub-exponential random variables.


 \begin{definition} [Sub-Exponential Random Variable \cite{wainwright2019high}]\label{def:subexp}A random variable $X$ with mean $\mu$ is sub-exponential  with parameter $(\tau^2,\nu) $ if: 
\begin{align*}
 \mathbb{E}[e^{s(X-\mu)}]\leq e^{\frac{s^2\tau^2}{2}},   \quad \forall s \text{ with }  \left|s \right|\leq \frac{1}{\nu}.
\end{align*}
\end{definition}
Any random variable  behaves sub-exponentially if it satisfies the Bernstein condition:
\begin{fact}[Bernstein Condition \cite{wainwright2019high}]
A random variable $X$ with mean $\mu$ and variance $\sigma^2$ satisfies the Bernstein condition if there exists $b>0$ such that 
\begin{align}
    \left |\mathbb{E}[(X-\mu)^k ] \right |\leq \frac{1}{2}k!\sigma^2b^{k-2} \quad \text{ for  } k\geq 2 . \label{eq:bern}
 \end{align}
 If $X$ satisfies \eqref{eq:bern}, then it is sub-exponential with parameter $\tau = \sqrt{2}\sigma ,\nu=2b$. 
 \end{fact}
 
 
 \begin{theorem}[Sub-Exponential Concentration Bound \cite{wainwright2019high}]\label{thm:bern}
 Let $X_1,\ldots ,X_n$ be independent sub-exponential random variables   each with  parameters $(\tau^2,\nu) $ and  $\mathbb{E}[X_i]=\mu.$ Then:
 \begin{align}
     P \left[\left|\frac{1}{n}\sum_i X_i -\mu \right|>t \right] \leq \exp\left(-\frac{n}{2}\min \left\{\frac{t^2}{\tau^2},\frac{t}{\nu}\right\} \right) \label{eq:bern2}.
 \end{align}
 
 \end{theorem}

We will also use: 
 \begin{theorem}[Isserlis's  Theorem \cite{wainwright2019high}]\label{thm:iss}
If $(X_1,X_2,\ldots,X_n)$  is a zero-mean multivariate normal random vector and $P_n^2$ is the set of all possible pairings of the set $\{1,\ldots,n\}$:
\begin{align*}
    \mathbb{E}[X_1\ldots X_n]= \sum_{p \in P^2_n} \prod_{(i,j)\in p} \mathbb{E}[X_iX_j]=\sum_{p \in P^2_n} \prod_{(i,j)\in p} Cov[X_i ,X_j].
\end{align*}
\end{theorem}

\subsection{Shortest Path Covariance Estimation}\label{sec:gs}
We consider the problem of estimating a positive semidefinite covariance matrix $\mathbf{\Sigma} \in \mathbb{R}^{d\times d}$ given $n$  independent samples $\x^{(1)},\ldots,\x^{(n)} \in \mathbb{R}^d$ drawn independently from the multivariate Gaussian distribution $\mathcal{N}(\mathbf{0},\mathbf{\Sigma})$,  with mean $\bv{0}$ and covariance $\bs{\Sigma}$. 
We assume that $\bs{\Sigma}$ is a shortest path covariance matrix (Definition \ref{def:graphStructure}) with the covariance between entries $i$ and $j$ only depends on the shortest path distance between nodes $v_i$ and $v_j$ in some graph $G$.
Our objective is to use samples from $\mathcal{N}(\bv 0, \bs \Sigma)$ to estimate $\Tilde{\bSigma} \in \mathbb{R}^{d \times d}$  such that with high probability ($\ge 1-\delta $  for some small $\delta>0$):
\begin{align}
    \left\|\bSigma -\Tilde{\bSigma}\right\|_2 \leq \epsilon 
    \left\| \bSigma \right\|_2. \label{obj}
\end{align}
In particular, we are interested in the sample complexity (both ESC and VSC) of the problem i.e. the minimum number of samples needed to achieve $ \epsilon$ error. We will not focus on computational complexity in particular, but the algorithms we present are computationally efficient, with low polynomial runtime.
\section{Main Results}\label{sec:main}
We now present our shortest path covariance estimation algorithm and give our main sample complexity bound. We start by introducing our primary technique for achieving low entry sample complexity -- sparse rulers on graph nodes.
\subsection{Sparse Rulers on Graphs}\label{sec:ruler}

We can see from Definition \ref{def:graphStructure} that in principle, to learn $\bs{\Sigma}$ it is enough to estimate the covariance $\bv{a}_s$ for all possible shortest path distances $s$ in the graph $G$. We will see how to do this by using just a small fraction of nodes in $G$ to represent all possible distances. 
%
As the nodes of the graph correspond to entries of the samples $\bv{x}^{(1)},\ldots, \bv{x}^{(n)}$, this allows us to access just a fraction of the entries in each sample and achieve entry sample complexity much lower than $d$ when estimating $\bs \Sigma$. 

We start with the simple observation that a graph with diameter $D$ has at least one pair of nodes $v_i,v_j$ whose shortest connecting path contains $D$ edges. The nodes along this path have shortest path distances ranging from $0$ to $D$. Further, we can measure all possible distances in this set  by choosing a subset of just $\Theta(\sqrt{D})$ nodes along this path according to a {sparse ruler}. 
Formally:


\begin{definition}
A subset $R \subseteq \{0,1,\ldots,D\}$ is a ruler if for all $s=\{0,\ldots, D \}$, there exist $j,k\in R $ so that $s=|j-k|$. We let $R_s= \{(j,k)\in R \times R :|j-k|=s \}$ be the ordered set of pairs in $R\times R$ with distance $s$. $R$ is called sparse if $|R| < D+1$.
\end{definition}

For any $D$, it is simple to construct a $\Theta(\sqrt{D})$ sparse ruler deterministically: 
\begin{proposition}\label{prop:sparse}
For any $D$, there is an explicit ruler $R$ of size $|R|= 2\lceil \sqrt{D} \rceil -1$.
\end{proposition}
\begin{proof} 
We can observe that for any $D$, the set $R= \{0,1,\ldots, \lceil \sqrt{D} \rceil  \} \cup \{D,D-\lceil \sqrt{D} \rceil,\ldots,D-(\lceil \sqrt{D}\rceil -2)\lceil \sqrt{D} \rceil \}$, is a ruler and has size $|R|=2\lceil \sqrt{D} \rceil -1$. 
\end{proof}
We note that while Proposition \ref{prop:sparse}  suffices for our results, a large body of work has studied sparse ruler design, aimed at improving the constant $2$ in front of  $\lceil D \rceil$ in the ruler size \cite{caratelli2011novel,qin2015generalized,cohen2019sparse,leech1956representation,wichmann2002note}.  Plugging in such optimized rulers will directly translate to constant factor improvements in entry sample complexity. 

\begin{figure}%
    \centering
    \subfloat[\label{fig:gph}]{{\includegraphics[width=3.2cm]{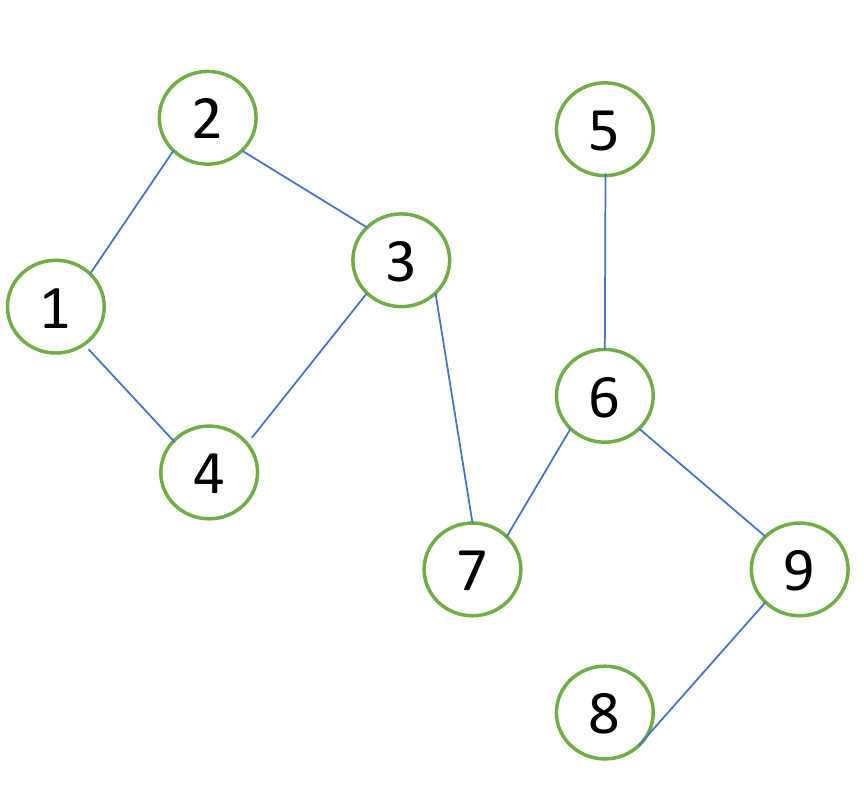} }}
    \qquad
    \subfloat[\label{fig:diam}]{{\includegraphics[width=3.2cm]{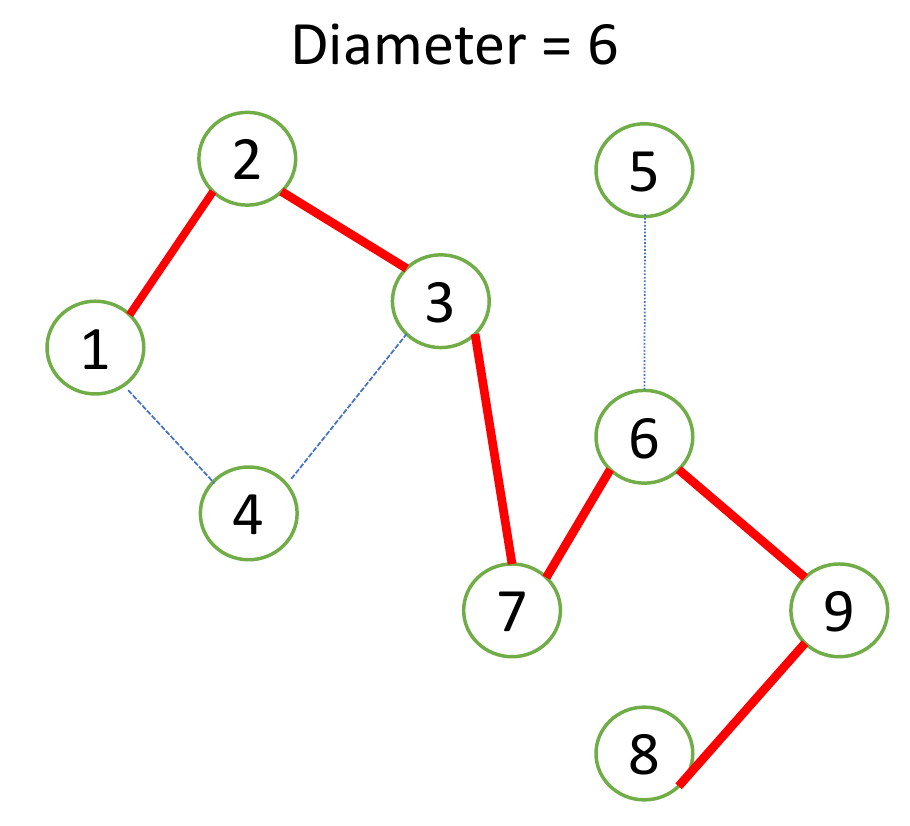} }}
    \qquad
    \subfloat[\label{fig:sparse}]{{\includegraphics[width=3.2cm]{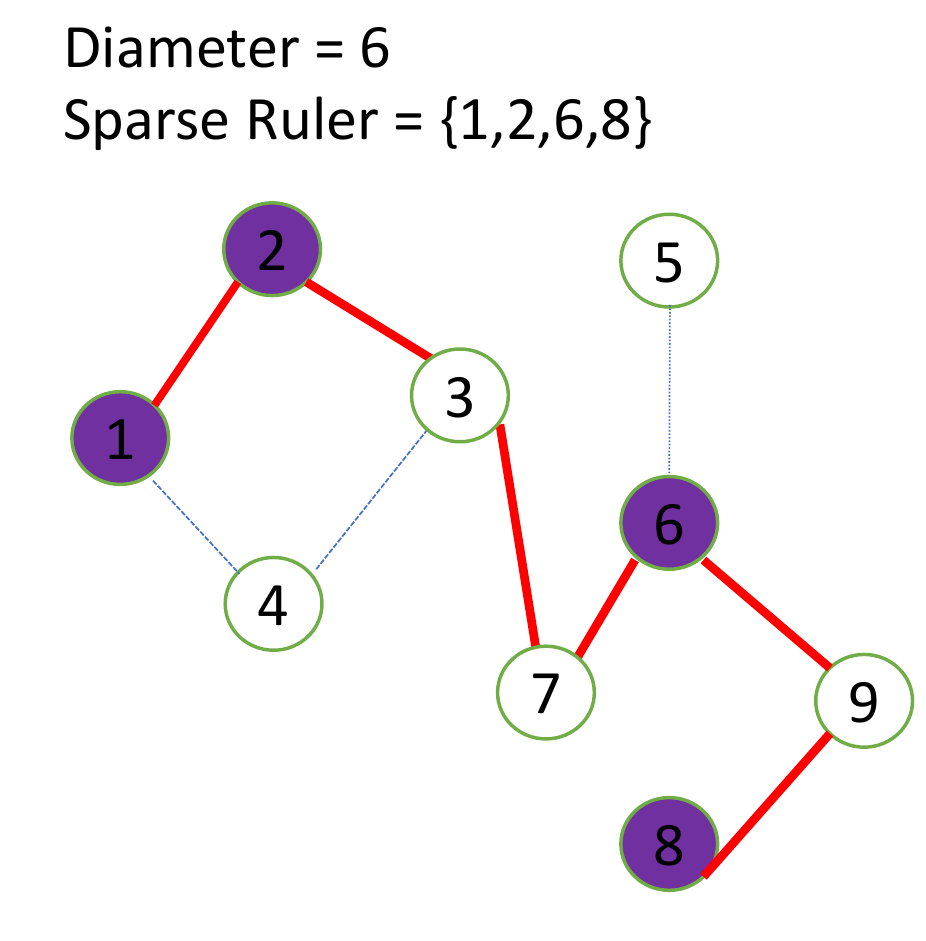} }}%
    \caption{\ref{fig:gph}: Underlying graph $G$. \ref{fig:diam}: The red path is a path with diameter $D=6$ edges. \ref{fig:sparse}: The purple highlighted nodes $\{1,2,6,8\}$ are the nodes in the sparse ruler, corresponding to positions $R = \{0,1,4,6\}$ on the path. The shortest path distances between these nodes cover all possible distances in $G$ since $R$ is a ruler for $D = 6$.}%
    \label{fig:graph}%
\end{figure}

In Figure~\ref{fig:graph} we illustrate how to use a sparse ruler to represent all shortest path distances on a graph using a small subset of nodes. The idea, formalized in Algorithm \ref{alg:sparse}, is simple: 1) label the nodes along any length-$D$ shortest path in order 2) take the nodes with labels corresponding to the indices in any sparse ruler $R$. The set of all shortest path distances between these $\Theta(\sqrt{D})$ nodes will include all distances in $\{0,\ldots,D\}$. Thus, we can use measurements at just these nodes to estimate all entries in $\bs{\Sigma}$. 

When e.g., $G$ is a clique, $\bs{\Sigma}$ has just two unique entries: one value $\bv{a}_0$ on diagonal and one value $\bv{a}_1$ off diagonal. We can just pick any two nodes to estimate these two values. When $G$ is just a path, with diameter $D=d-1$, a sparse ruler of $\Theta(\sqrt{d})$ entries can be applied to estimate the Toeplitz covariance matrix $\bs{\Sigma}$ \cite{toep}. Of course, while this technique will reduce our entry sample complexity, by taking measurements at fewer nodes, we obtain fewer overall samples. Thus, the challenge becomes bounding the required vector sample complexity of this method.

\begin{algorithm}[h]
\SetAlgoLined
 \textbf{Input: } Graph $G = (V,E)$.\\
  Calculate all pairs shortest path distances and the diameter $D$ of $G$.\\
  Let $P$ be the shortest path between any pair of nodes $v,u$ with shortest path distance $d(v,u) = D$.\\
  Let $v_0,v_1,\ldots, v_d$ be the nodes in $P$ listed in order from $v$ to $u$.\\
  Find a sparse ruler $R$ for $D$ (e.g., using Proposition \ref{prop:sparse}).\\
  Let $R^G \subseteq \{v_0,v_1,\ldots v_d\}$ be the set of nodes in $P$ indexed by the entries in $R$\\
  \noindent\textbf{Output: } $R^G \subseteq V$.
 \caption{Computing a Graph Sparse Ruler}
 \label{alg:sparse}
\end{algorithm}

\subsection{Shortest Path Covariance Estimation Algorithm}
As discussed, for a shortest path covariance matrix $\bs \Sigma$ with underlying graph $G$, for any $i,j$, $\bSigma_{i,j} = \ba_{s}$, where $s = d(v_i,v_j)$ is the shortest path distance between vertices $v_i$ and $v_j$ in $G$. Given a sparse ruler $R^G$ constructed as in Algorithm \ref{alg:sparse}, let:
\begin{align}
    R_s^G= \{(i,j)\in R^G \times R^G : d(v_i,v_j)=s\} \label{rs}.
\end{align}
For any $(i,j) \in R^G_s$ and $\x^{(l)} \sim \mathcal{N}(\bv 0, \bs \Sigma)$, $\mathbb{E}[\x^{(l)}_{i}\x^{(l)}_{j}] = \bv{a}_s$ -- that is, by considering the $i^{th}$ and $j^{th}$ entries of our sample, we obtain an unbiased estimate of the covariance at distance $s$. If we average $\x^{(l)}_{i}\x^{(l)}_{j}$ over sufficiently many independent samples $n$, we will thus obtain a good estimate $\bv{\tilde a}_s$ of $\bv{a}_s$. 
We can then construct an approximation to $\bs \Sigma$ using these covariance estimates.
 In particular, for $\bv{\tilde a} = [\bv{\tilde a}_0,\ldots,\bv{\tilde a}_D]$, let $\text{Graph-Cov}(G, \bv{\tilde a}) \in \mathbb{R}^{d \times d}$ be the shortest path covariance matrix with underlying graph $G$ and corresponding covariance values given by $\bv{\tilde a}$. Setting $\bs{\tilde \Sigma} = \text{Graph-Cov}(G, \bv{\tilde a})$, we will prove that $\bs{\tilde \Sigma}$ gives a good estimate of the true covariance $\bs \Sigma$.
 We formally describe our  method in Algorithm~\ref{alg:est}.

\begin{algorithm}[h]
\SetAlgoLined
 \textbf{Input: } Graph $G$ and independent samples $\x^{(1)},\ldots,\x^{(n)} \sim \mathcal{N}(0,\bSigma)$, where $\bs{\Sigma}$ is a shortest path covariance matrix (Def. \ref{def:graphStructure}) with underlying graph $G$.\\
Compute all pairs shortest path distances, diameter $D$, and a sparse ruler $R^G$ for $G$ using Algorithm \ref{alg:sparse}.\\
 \For {$s =0,\ldots, D $}{
 Choose any $(i,j)\in R^G_s$ (i.e., with $d(v_i,v_j) = s$).\\
 Set $\Tilde{\ba}_s = \frac{1}{n}\sum_{l=1}^n  \x^{(l)}_{i}\x^{(l)}_{j}$ . 
 }
 return $\Tilde{\bSigma}=\text{Graph-Cov}(\Tilde{\ba},G)$.
 \caption{Covariance Estimation with Sparse Ruler}
 \label{alg:est}
\end{algorithm}

\subsection{Sample Complexity Bound}

It is clear that the ESC of Algorithm 2 is $\Theta(\sqrt{D})$ where $D$ is the diameter of the underlying graph $G$. The challenge is in bounding the VSC $n$ necessary to estimate $\bs \Sigma$ up to $\epsilon$ error in the spectral norm. We provide this analysis in this section.

\begin{theorem}[Algorithm \ref{alg:est} Sample Complexity Bound]\label{thm:upper}
Consider any shortest path covariance matrix $\bSigma \in \mathbb{R}^{d \times d}$ with underlying graph $G$ and rank $r$. For any $\epsilon,\delta>0$, Algorithm \ref{alg:est} returns $\Tilde{\bSigma} $ such that with probability $\ge 1-\delta$,  $\|\bSigma- \Tilde{\bSigma}\|_2 \leq \epsilon\|\Tilde{\bSigma}\|_2 $, using entry sample complexity $\Theta(\sqrt{D})$ and vector sample complexity $n=O \left (\frac{r^2\log (\frac{D}{\delta} )}{\epsilon^2}\right )$,  where $D$ is the diameter of $G$.
\end{theorem}

\begin{proof}
Our proof will show that each covariance estimate $\bv{\tilde a}_s$ is close to the true covariance $\ba_s$, and in turn that $\bs{\tilde \Sigma}$ is close to $\bs{\Sigma}$.
First we prove that each term $\x^{(l)}_{i}\x^{(l)}_{j}$ for all $i,j \in [d] $ and $l \in [n]$ satisfies the Bernstein condition (Fact \ref{thm:bern}) and thus behaves sub-exponentially (Def. \ref{def:subexp}).

We first consider the moment $\left | \mathbb{E}\left [ \left (\x^{(l)}_{i}\x^{(l)}_{j} - \E[\x^{(l)}_{i}\x^{(l)}_j] \right)^k \right ] \right |$, which we must bound to give the Bernstein condition. Note that $\E[\x^{(l)}_{i}\x^{(l)}_j]  = \bSigma_{i,j}$. 
Thus, expanding this moment out yields  $2^k$ terms, each of the form   $\E\left [(\x^{(l)}_i\x^{(l)}_j)^m \right ] \cdot \bs{\Sigma}_{i,j}^{k-m}$. We can write:
\begin{align*}
  (\x^{(l)}_i\x^{(l)}_j)^m =\underbrace{(\x^{(l)}_i\ldots \x^{(l)}_i)}_{m \text{ times }}\underbrace{(\x^{(l)}_j\ldots \x^{(l)}_j)}_{m \text{ times }}. 
\end{align*}
Using Isserlis's Theorem (Theorem~\ref{thm:iss}) we thus have: 
\begin{align}
    \mathbb{E}[(\x^{(l)}_{i}\x^{(l)}_{j})^m]= \sum_{p \in P^2_{2m}} \prod_{(z,w)\in p} c_{z,w} 
\end{align}
where $c_{z,w} = \bs \Sigma_{i,i} = \bs \Sigma_{j,j} = \bv{a}_0$ if both $z$ and $w$ are either $\le k$ or $> k$, and $c_{z,w} = \bs{\Sigma}_{i,j}$ if $z \le k$ and $w > k$ or vice-versa. By the fact that $\bs \Sigma$ is positive semidefinite we have $|\bSigma_{i,j}| \le \bSigma_{i,i} + \bSigma_{j,j} \leq 2 \ba_0$. Thus, each term $c_{z,w}$ is bounded in magnitude by $2\ba_0$. Additionally, the number of possible pairing in $P_{2m}^2$ is $\frac{2m!}{m! \cdot 2^m} \le 2^m m! \le 2^m k!$. 
\begin{align*}
   \left | \mathbb{E}\left[(\x^{(l)}_i\x^{(l)}_j)^m \right] \right |
    & \leq 2^m k! (2\ba_0)^m \nonumber\\
    & \leq k!(4\ba_0)^m. 
\end{align*}
Thus we have each of the $2^k$ terms of $ \mathbb{E}\left [ \left (\x^{(l)}_{i}\x^{(l)}_{j} - \E[\x^{(l)}_{i}\x^{(l)}_j] \right)^k \right ]$, of the form $\E\left [(\x^{(l)}_i\x^{(l)}_j)^m\right ] \cdot \bs{\Sigma}_{i,j}^{k-m}$ bounded in magnitude by $k!(4\ba_0)^m \cdot \bs{\Sigma}_{i,j}^{k-m} \le k! (4\ba_0)^k$. So overall, via triangle inequality 
\begin{align}\label{bern}
\left |  \mathbb{E}\left [ \left (\x^{(l)}_{i}\x^{(l)}_{j} - \E[\x^{(l)}_{i}\x^{(l)}_j] \right)^k \right ] \right | \le 2^k \cdot k!\cdot(4\ba_0)^k = k!\cdot(8\ba_0)^k.
\end{align}

Now we can bound $\bv{a}_0$ (the diagonal value of $\bs \Sigma$) by using that $\bs \Sigma$ has rank $r$. In particular letting $\lambda_i(\bSigma)$ denote the $i^{th}$ eigenvalue of $\bSigma$:
\begin{align}
Tr(\bSigma) = \sum_{i=1}^r \lambda_i(\bSigma) & \Rightarrow d \cdot \ba_0 \leq r \left\|\bSigma \right\|_2 \Rightarrow \ba_0 \leq \frac{r}{d} \left\|\bSigma \right\|_2. \label{rank}
\end{align}
Plugging the bound of \eqref{rank} into \eqref{bern} we have:
\begin{align*}
    \mathbb{E}\left[(\x^{(l)}_i\x^{(l)}_j)^k \right] \leq k! \cdot \left(\frac{8r}{d} \left\|\bSigma \right\|_2 \right)^k \le \frac{1}{2}k! \cdot \left(\frac{8\sqrt{2} r}{d} \left\|\bSigma \right\|_2 \right)^2\left(\frac{8r}{d} \left\|\bSigma \right\|_2 \right)^{k-2}
\end{align*}
So each $(\x^{(l)}_{i}\x^{(l)}_{j})$ is sub-exponential with mean $\mu=\bSigma_{i,j}$ and  parameters $\tau= \nu=\frac{16 r}{d}\left\|\bSigma \right\|_2$ by the Bernstein condition (Theorem~\ref{thm:bern}).
Applying the sub-exponential concentration bound of Theorem \ref{thm:bern} we have for each estimate $\Tilde{\ba}_s$:
\begin{align*}
\mathbb{P} \left[ \left|\Tilde{\ba}_{s}-\ba_{s} \right|\geq t \right]
 &= P\left[\left|\frac{1}{n}\sum_{l=1}^n  \x^{(l)}_{i}\x^{(l)}_{j}-\ba_s \right| \geq t \right] \leq \exp \left(-\frac{n}{2}\min \left  \{\frac{t^2}{\tau^2},\frac{t}{\tau}\right \}\right).
\end{align*}

Now applying union bound over all $s=0,1,\ldots,D$ we have that the above bound holds with probability $(D+1)\exp \left(-\frac{n}{2}\min \{\frac{t^2}{\tau^2},\frac{t}{\tau}\}\right)$ for all estimates in $\Tilde{\ba}$ simultaneously. If we set $n \geq  2\log \left(\frac{D+1}{\delta}\right) \max \left \{\frac{\tau^2}{t^2},\frac{\tau}{t}\right \}$ we have:
\begin{align*}
  &  (D+1)\exp \left(-\frac{n}{2}\min \{\frac{t^2}{\tau^2},\frac{t}{\tau}\} \right) \leq \delta.
\end{align*}
Choosing $t= \frac{\epsilon \cdot \|\bSigma\|_2}{d}$ and recalling that $\tau=\frac{16r}{d}\left\|\bSigma \right\|_2 $, we thus have that for $n=O\left(\frac{r^2}{\epsilon^2} \cdot \log(\frac{D}{\delta}) \right)$, with probability $\ge 1-\delta$, s
$\left|\Tilde{\bSigma}_{i,j}-\bSigma_{i,j} \right| \le \frac{\epsilon \cdot \|\bSigma\|_2}{d}$ simultaneously for all $i,j$. Finally this entrywise bound gives: 
\begin{align*}
\|\Tilde{\bSigma}-\bSigma\|_2 &\leq  \|\Tilde{\bSigma}-\bSigma\|_F\\
 &= \sqrt{ \sum_{i,j} |\Tilde{\bSigma}_{i,j}-\bSigma_{i,j}|^2} \\
    & \leq d \cdot \frac{\epsilon \cdot \|\bSigma\|_2}{d} = \epsilon\|\bSigma\|_2,
\end{align*}
with probability at least $1-\delta$, completing the theorem.
\end{proof}
Plugging in full-rank $r =d$ to Theorem \ref{thm:upper} immediately gives:
\begin{corollary}\label{cor:lowrank}
For any shortest path covariance matrix $\bSigma$,  Algorithm \ref{alg:est} returns $\Tilde{\bSigma} $ such that $\|\bSigma- \Tilde{\bSigma}\|_2 \leq \epsilon\|\Tilde{\bSigma}\|_2 $ with probability $\ge 1-\delta$, using entry sample complexity $\Theta(\sqrt{D})$  and vector sample complexity $O(\frac{d^2\log (\frac{D}{\delta})}{\epsilon^2})$.
 \end{corollary}
\begin{remark}
Notice that the bound of the Theorem~\ref{thm:upper} actually upper bounds the Frobenius norm $\|\Tilde{\bSigma}-\bSigma\|_F$, which is only larger than the spectral norm. We believe that the given bound is tight for the Frobenius norm, and establishing a tighter bound for the spectral norm error (without going through the Frobenious norm) to match the lower bound of Section \ref{sec:lower} is an interesting problem. This is what is done, e.g., in \cite{toep}, which matches our lower bound in the special case when $G$ is a path and so $\bSigma$ is Toeplitz.
\end{remark}

\begin{remark}
If we consider a rank $r\leq \sqrt{d}$ Toeplitz covariance matrix (when $G$ is a path graph), the bound of Theorem \ref{thm:upper} slightly improves on the bound of Theorem 2.8 in \cite{toep}. We note that our approach is much simpler and more general -- using nothing specific about Toeplitz structure, beyond that it is a special case of a shortest path covariance.
\end{remark}

We can also give a simple extension to the case when $\bSigma$ is near low-rank in that its spectrum is dominated by at most $r < d$ large eigenvalues.
\begin{theorem}[Algorithm \ref{alg:est} Near Low-Rank Sample Complexity Bound]\label{thm:nearlow}
Consider any shortest path covariance matrix $\bSigma \in \mathbb{R}^{d \times d}$ with underlying graph $G$ which is near rank-$r$ i.e.,
\begin{align}
    Tr(\bSigma) \leq r \left\|\bSigma \right\|_2 + \xi\left\|\bSigma \right\|_F \label{eq: nearlow}
\end{align}
for small $\xi>0$. For any $\epsilon,\delta>0$, Algorithm \ref{alg:est} returns $\Tilde{\bSigma} $ such that with probability $\ge 1-\delta$,  $\|\bSigma- \Tilde{\bSigma}\|_2 \leq \epsilon\|\Tilde{\bSigma}\|_2 $, using entry sample complexity $\Theta(\sqrt{D})$ and vector sample complexity $n=O \left (\frac{\tilde{r}^2\log (\frac{D}{\delta} )}{\epsilon^2}\right )$,  where $D$ is the diameter of $G$ and $\tilde{r}= r + \frac{\xi\left\|\bSigma \right\|_F}{\left\|\bSigma \right\|_2 }$.
\end{theorem}
\begin{proof}
Consider the characterization of near low rank matrix in the equation~\eqref{eq: nearlow} and using  it in ~\eqref{rank} in the proof of Theorem~\ref{thm:upper} we have
\begin{align*}
     & Tr(\bSigma)  \leq r \left\|\bSigma \right\|_2 + \xi\left\|\bSigma \right\|_F  \\
 \Rightarrow &    d\ba_0  \leq r \left\|\bSigma \right\|_2 + \xi\left\|\bSigma \right\|_F \\
 \Rightarrow &  \ba_0  \leq \frac{1}{d}\left(r + \frac{\xi\left\|\bSigma \right\|_F}{ \left\|\bSigma \right\|_2} \right) \left\|\bSigma \right\|_2 \\
 \Rightarrow &  \ba_0  \leq \frac{\tilde{r}}{d}\left\|\bSigma \right\|_2.
\end{align*}

 The rest follows from the proof of Theorem~\ref{thm:upper}.
\end{proof}

\section{Lower Bound}\label{sec:lower}
We now  prove a simple lower bound on the sample complexity of shortest path covariance estimation when a fixed subset of entries $S$ (e.g., corresponding to a sparse ruler) are read in each sample. We prove the lower bound via a reduction to a property testing problem. 
Consider the class of distributions  $\mathcal{C}= \left\{\mathcal{N}(0,\bSigma): \bSigma \text{ is a shortest path covariance} \right\}$ and a particular property $\mathcal{P}$, which is defined by a subset of distributions over $\R^d$.
Consider a (deterministic) testing algorithm $\mathcal{T}_n$ for $\mathcal{P}$ which takes $\x^{(1)},\ldots, \x^{(n)}$ independent samples from some  $C \in \mathcal{C}$ as input and outputs \{ACCEPT, REJECT\}. The testing algorithm fails if the algorithm outputs ACCEPT on a \textit{no} instance (i.e the property does not hold for $C$) and REJECT on a \textit{yes} instance (i.e the property  holds for $C$). Let $\mathcal{D}_1 \subset \mathcal{C}$ and $\mathcal{D}_2 \subset \mathcal{C} $ be the families of \textit{yes} and \textit{no} instances respectively. 
We can formally bound the failure probability of any such testing algorithm using e.g., Corollary D.5.8 of \cite{canonne2015survey}:
\begin{theorem}\label{thm:lecam}
Fix $\epsilon \in (0,1)$ and a property $\mathcal{P}$. Let $\mathcal{D}_1,\mathcal{D}_2 \subseteq \mathcal{C}$  be the families of \textit{yes} and \textit{no} instances respectively such that $\mathcal{D}_1 \subseteq \mathcal{P}$ , while any $D \in \mathcal{D}_2$ and $D' \in \mathcal{P}$ have $\left\|D-{D}' \right\|_{TV} \geq \epsilon$. For all $n\geq 1$,
\begin{align}
    \inf_{\text{algorithms }\mathcal{T}_n} \sup_{C\in \mathcal{C}} \displaystyle\mathbb{P}_{\x^{(1)},\ldots, \x^{(n)} \sim C }\left[\mathcal{T}_n(\x^{(1)},\ldots, \x^{(n)}) \text{ fails}) \right] \geq \frac{1}{2} \left(1- \inf_{ \substack{D_1 \in conv_n(\mathcal{D}_1) \\ D_2 \in conv_n(\mathcal{D}_2)} }\left\|D_1-D_2\right\|_{TV} \right) , \label{lblecam}
\end{align}
where $conv_n$ denotes the convex hull of a family of $n$-fold  product distributions defined as:
\begin{align*}
  conv_n(\mathcal{D})\eqdef \left\{ \sum_{i=1}^k \alpha_iD_i^{\otimes n}: k\geq 1 ;D_1,\ldots,D_k \in \mathcal{D};\alpha_1,\ldots\alpha_k \geq 0,\sum_{i=1}^k \alpha_i=1 \right\}.  
\end{align*}

\end{theorem}

Now, if $n$ and $\mathcal{D}_1,\mathcal{D}_2$ satisfy $\inf_{ \substack{D_1 \in conv_n(\mathcal{D}_1) \\ D_2 \in conv_n(\mathcal{D}_2)} }\left\|D_1-D_2\right\|_{TV} \leq \frac{1}{3}$ in equation~\eqref{lblecam}, the above Theorem~\ref{thm:lecam} (along with Yao's principle) implies a lower bound of $\Omega(n)$ for the testing algorithm for any possibly randomized algorithm that that can only fail with probability $\frac{1}{3}$. 

\begin{theorem}[Sample Complexity Lower Bound] \label{thm:lower}
Any (possibly randomized) algorithm that is given  $n$ independent samples from the multivariate Gaussian distribution $\mathcal{N}(\bv 0,\bSigma)$ where $\bSigma$ is a shortest path covariance matrix, reads a fixed subset $S$ of $s$ entries from each sample, and returns $\bv{\tilde \Sigma}$ such that $\|\bSigma- \Tilde{\bSigma}\|_2 \leq \epsilon\|\Tilde{\bSigma}\|_2 $ with probability $\ge 2/3$ requires $n = \Omega \left (\frac{d^2}{s^2 \epsilon^2} \right )$ vector samples. When $s = O(\sqrt{D})$, as when using a graph sparse ruler, this gives $n = \Omega \left (\frac{d^2}{D \epsilon^2} \right )$.
\end{theorem}
\begin{proof}

We consider two distributions $P_1=\mathcal{N}(0,\bSigma_1)$ and $P_2=\mathcal{N}(0,\bSigma_2)$ where $\bSigma_1 =\mathbf{I}$ and $\bSigma_2 =\mathbf{I}+ \frac{\epsilon}{d}\J$, with $\J$ denoting the  all ones matrix. Note that both are shortest path covariance matrices, valid for any graph $G$ (where covariance at every distance more than zero is either $0$ or $\epsilon/d$ ).
We make the following observation:
\begin{align*}
   \left\|\bSigma_1 -\bSigma_2 \right\|_2^2 &=\left\|\frac{\epsilon}{d}\J\right\|^2_2\\
   &= \left\|\frac{\epsilon}{d}\J\right\|^2_F  && \text {(as the matrix has rank $1$)}\\
  & = \frac{\epsilon^2d^2}{d^2} =\epsilon^2.
\end{align*}
Thus we have $\left\|\bSigma_1 -\bSigma_2 \right\|_2=\epsilon$. Also notice that $\left\|\bSigma_1 \right\|_2^2=1$ and $\left\|\bSigma_2 \right\|_2^2 =1 + \epsilon^2$ are both upper bounded by $2$ provided $\epsilon<1$ which implies $\left\|\bSigma_1 -\bSigma_2 \right\|_2 \geq \frac{\epsilon}{2}\left\|\bSigma_i \right\|_2 $ for $i=1,2$.

Now if we have an learning algorithm that can estimate a given shortest path covariance matrix $\bSigma$ within error $\frac{\epsilon}{2}\left\|\bSigma \right\|_2$, we can distinguish between the two distributions with covariance $\bSigma_1$ and  $\bSigma_2$. If the problem of distinguishing between these two distribution requires $n $ samples then the estimation problem requires at least $n$ samples.

Here we have access to data samples according to a fixed subset of $s$ entries $S$. Thus, our input is just $n$ $s$-variate Gaussian random vectors from either  $\mathcal{N}(0,[\bSigma_1]_S)$ or $\mathcal{N}(0,[\bSigma_2]_S)$ where the subscript $S$ denote the principal sub-matrix corresponding to the set $S$. Here $[\bSigma_1]_S=\mathbf{I}_S$ is an $s \times s$ identity matrix and $[\bSigma_2]_S=\mathbf{I}_S + \frac{\epsilon}{d}\J_S$ where $\J_S$ is the $s \times s$  all ones matrix.
 
Now we consider ${P}_{1,S}= \mathcal{N}(0,[\bSigma_1]_S) $ and ${P}_{2,S}= \mathcal{N}(0,[\bSigma_2]_S) $ and the property $\mathcal{P}$ that the distribution is  $ \mathcal{N}(0,[\bSigma_1]_S)$, which is equivalent to the property that the full distribution is $ \mathcal{N}(0,[\bSigma_1])$.
To apply Theorem~\ref{thm:lecam}, we need to show that $\left\|D_1-D_2 \right\|_{TV} \leq \frac{1}{3}$ for some $D_1$ and $D_2$ in the convex hulls of $n$-fold  product distributions of ${P}_{1,S}= \mathcal{N}(0,[\bSigma_1]_S) $ and ${P}_{2,S}= \mathcal{N}(0,[\bSigma_2]_S) $ respectively. Since these sets are just singletons, this means we must show the bound for $D_1 ={P}_{1,S}^{\otimes n}$ and $D_2 ={P}_{2,S}^{\otimes n}$.

%
%
Now we provide an upper bound to the total variation distance using Pinsker's inequality.

\begin{fact}[Pinsker's Inequality)\cite{diakonikolas2019robust}]\label{fact:pink}
For any two distributions $P$ and $Q$ we have 
\begin{align*}
    \left\|P-Q \right\|^2_{TV} \leq  \frac{1}{2} KL(P\|Q).
\end{align*}
And for  product distributions $P^n=P_1\times \ldots\times P_n$ and $Q^n=Q_1\times \ldots\times Q_n$, we have 
    \begin{align*}
    \left\|P^n-Q^n \right\|^2_{TV} \leq  \frac{1}{2}\sum_{i=1}^n KL(P_i\|Q_i).
\end{align*}
\end{fact}
We bound the total variation distance between $D_1$ and $D_2$ using Fact~\ref{fact:pink} 
\begin{align}
\left\|D_1-D_2 \right\|^2_{TV} = \left\|{P}_{1,S}^{\otimes n}-{P}_{2,S}^{\otimes n}\right\|^2_{TV} \le \frac{n}{2} KL({P}_{1,S}\| {P}_{2,S}) \label{p1}.
\end{align}

\begin{fact}[KL divergence between multivariate Gaussians\cite{diakonikolas2019robust}]\label{fact:KLg}
The KL-divergence between two $d$-dimensional multivariate Gaussian distributions $\mathcal{N}(0,\bSigma_1)$ and $\mathcal{N}(0,\bSigma_2)$ is:
\begin{align*}
    KL(\mathcal{N}(0,\bSigma_1)\|\mathcal{N}(0,\bSigma_2)) = \frac{1}{2}\left[Tr\left(\bSigma_2^{-1}\bSigma_1 \right) -d -\log \left(\frac{det(\bSigma_1)}{det(\bSigma_2)} \right) \right].
\end{align*}
\end{fact}
Now we bound the term in equation ~\eqref{p1} using Fact~\ref{fact:KLg} with $\bSigma_1=\mathbf{I}_S $ and $\bSigma_2=\mathbf{I}_S + \frac{\epsilon}{d}\J_S$
\begin{align*}
    KL({P}_{1,S}\| {P}_{2,S}) & = \frac{1}{2}\left[ \log \left(\frac{det(\mathbf{I}_S + \frac{\epsilon}{d}\J_S)}{det(\mathbf{I}_S )} \right)-s +Tr\left((\mathbf{I}_S + \frac{\epsilon}{d}\J_S)^{-1}\mathbf{I}_S \right) \right].
\end{align*}
Now the matrix $(\mathbf{I}_S + \frac{\epsilon}{d}\J_S)$ has the top eigen-value $1+\frac{\epsilon}{d}s $ and the rest are $1$. As determinant is the product of the eigen-values of the matrix, we have the determinant to be $1+\frac{\epsilon}{d}s $. Using Sherman-Morrison-Woodbury Inversion lemma we have $  (\mathbf{I}_S + \frac{\epsilon}{d}\J_S)^{-1} = \mathbf{I}_S - \frac{\frac{\epsilon}{d}}{1+\frac{\epsilon}{d}s}\J_S$. So $ Tr\left((\mathbf{I}_S + \frac{\epsilon}{d}\J_S)^{-1}\mathbf{I}_S \right)$ is $s(1- \frac{\frac{\epsilon}{d}}{1+\frac{\epsilon}{d}s})$. We have
\begin{align*}
    KL({P}_{1,S}\| {P}_{2,S}) & = \frac{1}{2}\left[ \log \left(1+ \frac{\epsilon}{d}s\right) -s + s(1- \frac{\frac{\epsilon}{d}}{1+\frac{\epsilon}{d}s}) \right] \\
    &\leq  \frac{1}{2} \left[ \frac{\epsilon}{d}s - \frac{\frac{\epsilon}{d}s}{1+\frac{\epsilon}{d}s} \right]\\
    & \leq \frac{\epsilon^2 s^2}{d^2}
\end{align*}
Now plugging the value in equation~\eqref{p1}, we have $\left\|D_1-D_2 \right\|^2_{TV} =  O\left(n\frac{\epsilon^2 s^2}{d^2} \right)$. So we need $n= \Omega\left(\frac{d^2}{\epsilon^2 s^2}\right)$ to make value large so that the failure probability can be small.

\end{proof}

\subsection{Star Graphs}\label{sec:star}
We now study the shortest path covariance matrix of a star graph  and show that the sample complexity for estimation  matches the lower bound of Theorem \ref{thm:lower}. 
This extends the result of \cite{toep} for Toeplitz matrices and an interesting direction is to see if we can close the gap between the upper and lower bounds for more general graph structures.

We consider a star with $l$ branches with $\Delta$ nodes and  one center node. So we have $\Delta l+1=d$.

\begin{figure}[h]
    \centering
    \subfloat[\label{fig:sgph}]{{\includegraphics[width=4cm]{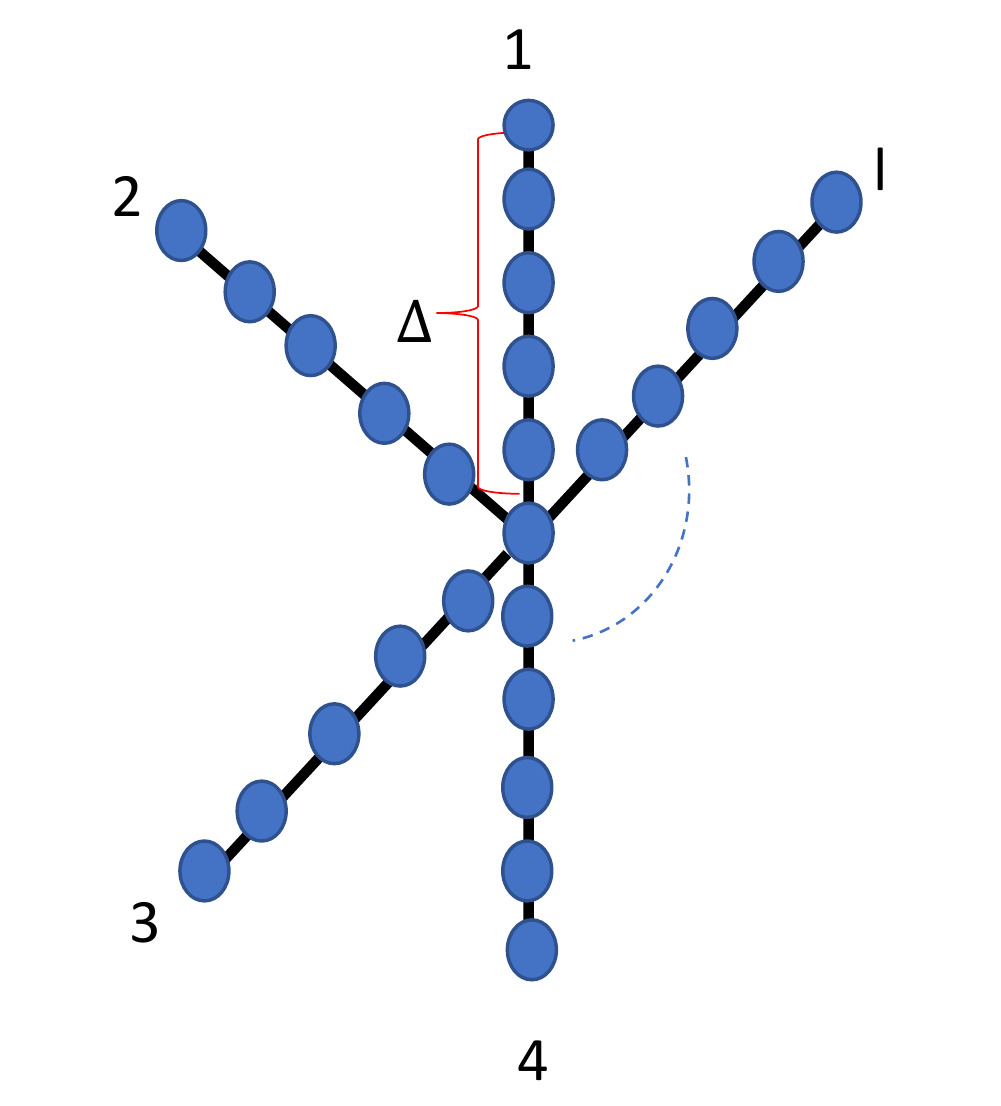} }}
    \qquad
    \subfloat[\label{fig:matstar}]{{\includegraphics[height=5cm,width=7cm]{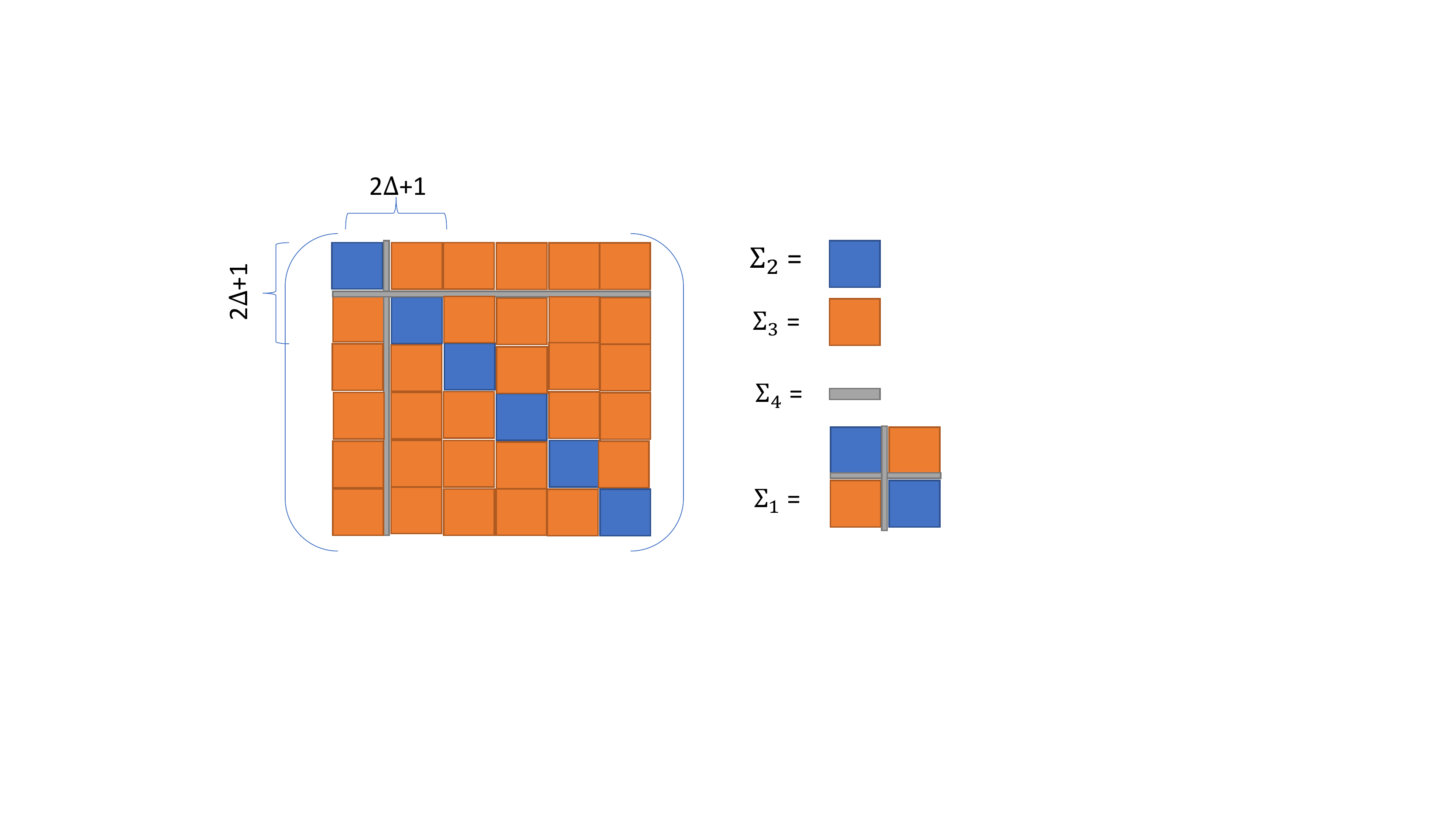} }}
    \caption{\ref{fig:gph}: Star graph. \ref{fig:matstar}: the corresponding covariance matrix.}%
    \label{fig:star}%
\end{figure}

\begin{theorem}[Sample Complexity Upper Bound for Star Graph] \label{thm:star}
Consider the shortest path covariance matrix  $\bSigma \in  \mathbb{R}^{d\times d}$ with underlying graph that is a star  with $l$ branches and $\Delta$ nodes in each branch such that $l\Delta+1=d$ (see Figure \ref{fig:star}). For any $\epsilon,\delta>0$, using entry sample $\Theta(\sqrt{\Delta})$ and vector sample complexity $n= \Tilde{O}\left(\frac{d^2}{\Delta\epsilon^2} \right)$, an approximate covariance $\Tilde{\bSigma}$ can be estimated such that $\left\|\Tilde{\bSigma} -\bSigma \right\|_2 \leq \epsilon \left\|\bSigma \right\|_2$ with probability $\ge 1-\delta$. 
\end{theorem}
\begin{proof}

By exploiting the structure and shortest path distance property, the covariance matrix $\bSigma$ can be thought to be made of different blocks as shown in Figure ~\ref{fig:matstar}. The blocks are: 
\begin{itemize}
    \item The covariance of an individual branch  (in  Figure \ref{fig:matstar} this is drawn as a blue square and denoted by the $\Delta \times \Delta$ matrix $\bSigma_2$).
    \item The covariance between a pair of branches (in  Figure \ref{fig:matstar} this is drawn as an orange square and denoted by the $\Delta \times \Delta$ matrix $\bSigma_3$).
    \item The covariance between the center node and a branch  (in  Figure \ref{fig:matstar} this is drawn as a grey rectangle and denoted by the $\Delta \times 1$ matrix $\bSigma_4$ ).
\end{itemize}
Now we consider the path (corresponding to the diameter of the star graph) with branches 1 and 2 along with the center node. The covariance of this path is a $(2\Delta+1) \times (2\Delta+1)$ Toeplitz matrix (denoted by $\bSigma_1$ in  Figure~\ref{fig:matstar}). The matrices $\bSigma_2,\bSigma_3,\bSigma_4$ are sub-matrices  of $\bSigma_1$. So given an estimate $\Tilde{\bSigma_1}$, the corresponding submatrices $\Tilde{\bSigma_2},\Tilde{\bSigma_3},\Tilde{\bSigma_4}$ all have spectral norm error upper bounded by $\left\|\bSigma_i - \Tilde{\bSigma_i} \right\|_2 \le \left\|\bSigma_1 - \Tilde{\bSigma_1} \right\|_2$. According to  Algorithm ~\ref{alg:est}, we apply a sparse ruler on the path which results in entry sample complexity $O(\sqrt{\Delta})$. As the matrix $\bSigma_1$ is Toeplitz, following the algorithm and analysis of \cite{toep}, we can have an estimate $\Tilde{\bSigma}_1$ such that $\left\|\Tilde{\bSigma}_1 -\bSigma_1 \right\|_2 \leq \epsilon \left\| \bSigma_1 \right\|_2$ with  $\Tilde{O}\left(\frac{\Delta}{\epsilon^2} \right)$ vector samples. The estimation of $\bSigma_1$ leads to the estimation of the whole matrix $\bSigma $, which is composed of blocks of type  $\bSigma_2,\bSigma_3,\bSigma_4$. Specifically, $\bSigma$ can be decomposed into $O\left(\frac{d^2}{\Delta^2} \right)$ blocks. The error matrix $\Tilde{\bSigma}- \bSigma$ is thus also composed of this many blocks, each with error bounded by $\left\| \Tilde{\bSigma}_i - \bSigma_i \right\|_2 \le \left\| \Tilde{\bSigma}_1 - \bSigma_1 \right\|_2 \le \epsilon \norm{\bSigma_1}_2$ for $i \in \{2,3,4\}$. Thus, using that the squared spectral norm of a block matrix is bounded by the sum of squared spectral norms of its blocks \cite{audenaert2006norm},
%
\begin{align*}
    \left\|\Tilde{\bSigma}- \bSigma \right\|_2^2 \leq c_1\frac{d^2}{\Delta^2}\epsilon^2 \left\| \Tilde{\bSigma}_1 - \bSigma_1 \right\|_2^2 \leq c_1\frac{d^2}{\Delta^2}\epsilon^2 \left\| \bSigma_1 \right\|_2^2 \le c_1\frac{d^2}{\Delta^2}\epsilon^2 \left\| \bSigma \right\|_2^2.
\end{align*}
Now we consider $\epsilon_1= \frac{\Delta}{\sqrt{c_1} d}\epsilon$. With $\Tilde{O}(\frac{\Delta}{\epsilon_1^2}) = \Tilde{O}\left(\frac{d^2}{\Delta \epsilon^2} \right) $ samples we have from the above bound that $\left\|\Tilde{\bSigma}- \bSigma \right\|_2 \leq \epsilon\left\| \bSigma \right\|_2$. As the diameter of the graph is $D=2\Delta+1$, we have overall  sample complexity $\Tilde{O}\left(\frac{d^2}{D \epsilon^2} \right)$ which matches the lower bound.
\end{proof}

\section{Conclusion and Future Work}\label{sec:conclusion}

We have studied the sample complexity of graph-structured covariance estimation, specifically focusing on shortest path covariance matrices. We have established an entry sample complexity bound depending on the square root of the underlying graph diameter using a generalization of sparse rulers to shortest path distances. We have also given a bound on the required vector sample complexity of our method, and a near matching lower bound in the case when the diameter $D$ is small. Our work leaves open a number of open questions.

\begin{enumerate}
\item Our upper and lower bounds (Theorems \ref{thm:upper} and \ref{thm:lower}) differ by a $D$ factor. Thus is conceivable that our vector sample complexity when using a $\Theta(\sqrt{D})$ sparse ruler can be improved from $\tilde O\left (\frac{d^2}{\epsilon^2}\right )$ to $\tilde O\left (\frac{d^2}{\epsilon^2 D}\right )$. When $D$ is small, this improvement is minor. However, in important special cases, such as when $\bSigma$ is Toeplitz and $D = d-1$, it is a major gap. The work of \cite{toep} closes the gap when $\bSigma$ is Toeplitz, however it heavily relies on the Fourier analytic structure of Toeplitz matrices, which does not extend to more general graph classes. We believe that matching our lower bound for more general graph classes is an interesting challenge that will require a deeper understanding of the structure of shortest path covariance matrices. Our bound for the star graph in Section \ref{sec:star} is a first step.
\item We give an algorithm with optimal entry sample complexity $\Theta(\sqrt{D})$ and polynomially bounded vector sample complexity $\tilde O(d^2/\epsilon^2)$. At the other extreme, if we use full entry sample complexity $d$, then standard matrix concentration bounds show that $\bSigma$ can be estimated with $O(d/\epsilon^2)$ vector sample complexity \cite{vershynin2018high}. Can we obtain a smooth tradeoff between these measures? This has been done again in the special case of Toeplitz covariance approximation \cite{toep}. Extending the results to general graphs seems to require overcoming two challenges: (1) we need to design graph sparse rulers that measure a larger number of entries and in turn obtain multiple estimates of the covariance $\bv{a}_s$ at each distance $s$ and (2) we need to understand how to bound the correlations between these measurements, to show that they actually lead to lower vector sample complexity. Challenge (2) is likely related to the general goal of improving the vector sample complexity to match our lower bound.
\item Our work specifically applies to shortest path covariance matrices, however an interesting question is if low entry sample complexity using sparse rulers or other techniques can be achieved for other graph-structure covarianced matrices, such as those arising in graph signal processing \cite{shuman2013emerging,ortega2018graph} and Gaussian Markov random fields \cite{rue2005gaussian,uhler2017gaussian}. Some work has be done in this direction \cite{graph}, and obtaining rigorous sample complexity bounds would be very interesting.
 In general, can one characterize the class of covariance matrices for which $o(d)$ entry sample complexity can be obtained? 
\end{enumerate}

\bibliographystyle{alpha}
\bibliography{covref}

\end{document}